\pdfoutput=1
\documentclass{article}

\usepackage[preprint]{cpal_2026}

\usepackage{amsmath,amssymb,amsthm}
\usepackage{url}
\usepackage{microtype}
\usepackage[T1]{fontenc}
\usepackage{lmodern}
\usepackage{graphicx}
\usepackage{booktabs}
\usepackage[numbers]{natbib}

\newtheorem{proposition}{Proposition}
\newtheorem{remark}{Remark}

\title{Evidence Slopes and Effective Dimension in Singular Linear Models}

\author{%
  Kalyaan Rao \\
  New York University, Center for Data Science \\
  \texttt{kmr9800@nyu.edu}
}

\begin{document}

\maketitle

\begin{abstract}
Bayesian model selection commonly relies on Laplace's approximation or BIC, which assume that the effective model dimension equals the number of parameters \citep{schwarz1978estimating,kass1995bayes}. Singular learning theory replaces this by the real log canonical threshold (RLCT), an effective dimension that can be strictly smaller in overparameterised or low-rank models \citep{watanabe2009algebraic,watanabe2018mathematical}. We study simple linear--Gaussian rank models where RLCT is analytically tractable and the exact marginal likelihood is available in closed form. In this setting, we show theoretically and empirically that the error of Laplace/BIC grows linearly with $(d/2-\lambda)\log n$, while an RLCT-aware correction recovers the correct slope and is invariant to overcomplete reparametrisations with the same data subspace. Our results provide a concrete, finite-sample characterisation of ``Laplace failure'' in singular models and a simple slope-based estimator of RLCT in textbook linear settings.
\end{abstract}

\section{Introduction}

Bayesian evidence (marginal likelihood) underlies many model selection procedures, and Laplace or BIC are the standard approximations in practice \citep{tierney1986accurate,schwarz1978estimating,kass1995bayes}.
These approximations implicitly assume regular models where the Fisher information is full rank and the effective dimension equals the parameter count $d$.

Many overparameterised models (low-rank regression, dictionary learning, neural networks) are singular: parameters are redundant, likelihoods have flat directions, and the usual $d/2$ penalty is incorrect.
Singular learning theory (SLT) shows that in such models the marginal likelihood is instead governed by the real log canonical threshold (RLCT) $\lambda$, with
\[
\log p(D_n) = \log p(D_n \mid \theta^\star) - \lambda \log n + O(\log\log n),
\]
where $\theta^\star$ is any parameter minimising the Kullback--Leibler divergence to the true data-generating distribution \citep{watanabe2009algebraic,watanabe2018mathematical}.
Using $d/2$ in place of $\lambda$ therefore mis-specifies the complexity term in $\log p(D_n)$ and can bias model comparison.

We revisit this gap in a deliberately simple setting where both RLCT and the exact evidence are tractable: linear--Gaussian rank models and linear subspace (dictionary) models with Gaussian priors, closely related to probabilistic PCA and factor analysis \citep{tipping1999probabilistic,ghahramani1997factor}.

Concretely, our contributions are:
\begin{itemize}
\item For rank-deficient linear--Gaussian regression, we derive the exact marginal likelihood in closed form and show that the RLCT is $\lambda(r)=r/2$, where $r$ is the intrinsic rank.
\item We prove that, in these models, the Laplace/BIC approximation incurs an error whose leading term is $\bigl(d/2-\lambda(r)\bigr)\log n = (d-r)\log n/2$, and we show empirically that a simple RLCT-aware correction removes this bias.
\item In a linear dictionary (subspace) model, we show that RLCT and evidence are invariant under overcomplete reparametrisations with the same span, while BIC is not; experiments with eigenspectra and exact evidences illustrate this representation non-invariance.
\item We interpret the slope of the evidence versus $\log n$ as a practical, finite-sample estimator of RLCT in these textbook linear settings, providing a simple way to ``see'' effective dimension from data.
\end{itemize}

\section{Related Work}

\paragraph{Laplace approximation and BIC.}
Laplace's method and the Bayesian Information Criterion (BIC) are standard tools for approximating marginal likelihoods in regular parametric models \citep{tierney1986accurate,schwarz1978estimating,kass1995bayes}.
Under smoothness and identifiability assumptions, the posterior concentrates at rate $n^{-1/2}$ around a unique maximum likelihood or MAP estimate $\hat\theta_n$, and the log evidence admits the expansion
\begin{equation}
  \log p(D_n)
  = \log p(D_n \mid \hat\theta_n) - \frac{d}{2}\log n + O(1),
\end{equation}
where $d$ is the parameter dimension.
BIC keeps only the data-fit term and the $d/2 \log n$ penalty and is widely used for model selection in regression, time series, and latent variable models.

\paragraph{Singular learning theory.}
Many modern models violate the regularity assumptions behind Laplace/BIC:
mixture models, neural networks, low-rank and overcomplete representations all exhibit redundant parameters and flat directions in the likelihood.
Singular learning theory replaces the dimension $d$ by the real log canonical threshold (RLCT) $\lambda$, a birational invariant that quantifies the local singularity of the likelihood at Kullback--Leibler minimisers \citep{watanabe2009algebraic,watanabe2018mathematical}.
In such models the evidence behaves as
\begin{equation}
  \log p(D_n)
  = \log p(D_n \mid \theta^\star) - \lambda \log n + (m-1)\log\log n + O(1),
\end{equation}
where $\theta^\star$ is any KL minimiser, $\lambda>0$ is the RLCT, and $m \in \mathbb{N}$ is its multiplicity.
For regular models $\lambda=d/2$ and $m=1$, recovering the classical Laplace/BIC expansion; in singular models typically $\lambda<d/2$.
Analytic RLCTs have been derived for certain mixture and rank-constrained models \citep{watanabe2001algebraic,watanabe2009algebraic,aoyagi2010asymptotic}, but closed-form results are rare beyond simple families.

\paragraph{Empirical flatness and effective dimension.}
Several works propose empirical surrogates for model ``flatness'' or effective dimension, based on the Hessian spectrum, Fisher information, or PAC-Bayesian bounds \citep{hochreiter1997flat,keskar2017large,liang2019fisher,dziugaite2017computing}.
These proxies often correlate with generalisation in overparameterised networks but are heuristic and not directly grounded in the asymptotic marginal likelihood.
Closer to SLT, some recent work attempts to estimate RLCT-like exponents from training curves or posterior samples, typically without access to ground-truth RLCT or exact evidences \citep{matsuda2019information,rissanen2010mdl}.
Our work differs in that we stay within a linear--Gaussian family where RLCT and the exact marginal likelihood are analytically tractable, and we compare these quantities directly to Laplace/BIC.

\paragraph{Low-rank and overparameterised models.}
Low-rank regression, factor analysis, and overcomplete dictionaries are prototypical examples of models that use a low-dimensional subspace but potentially many parameters to represent it \citep{ghahramani1997factor,tipping1999probabilistic,olshausen1997sparse}.
Model selection in these families is often carried out using BIC-type penalties on the raw number of parameters.
From an SLT perspective, however, the effective complexity should depend on the intrinsic rank (the dimension of the subspace) rather than on the number of coordinates used to represent it.
The linear--Gaussian models we study provide a setting where this distinction can be made explicit and quantified in terms of marginal likelihood.

\section{Background and Problem Setup}

\subsection{Bayesian evidence and Laplace approximation}

Let $D_n = \{(x_i,y_i)\}_{i=1}^n$ be data, $p(y\mid x,\theta)$ a likelihood, and $\pi(\theta)$ a prior on parameters $\theta \in \mathbb{R}^d$.
The marginal likelihood (evidence) of the model is
\begin{equation}
  Z_n := p(D_n)
  = \int p(D_n \mid \theta)\,\pi(\theta)\,d\theta
  = \int \prod_{i=1}^n p(y_i \mid x_i,\theta)\,\pi(\theta)\,d\theta.
\end{equation}
Bayesian model comparison typically prefers models with larger $Z_n$.

In a regular model, the posterior concentrates in a neighbourhood of a unique maximiser $\hat\theta_n$ of the log posterior as $n\to\infty$.
A second-order Taylor expansion around $\hat\theta_n$ yields Laplace's approximation \citep{tierney1986accurate}
\begin{equation}
  \log Z_n
  \approx \log p(D_n \mid \hat\theta_n) + \log \pi(\hat\theta_n)
  + \frac{d}{2}\log(2\pi) - \frac{1}{2}\log\det H_n,
\end{equation}
where $H_n$ is the Hessian of the negative log posterior at $\hat\theta_n$.
Under usual regularity conditions, $H_n = n I(\theta^\star) + o_p(n)$, where $I(\theta^\star)$ is the Fisher information at a KL minimiser $\theta^\star$, and $\det H_n \asymp n^d$, giving
\begin{equation}
  \log Z_n
  = \log p(D_n \mid \hat\theta_n) - \frac{d}{2}\log n + O_p(1).
\end{equation}
BIC discards $\log \pi(\hat\theta_n)$ and constant terms and keeps exactly this $d/2 \log n$ penalty \citep{schwarz1978estimating}.

\subsection{Singular learning theory and RLCT}

In singular models, multiple parameter values can represent the same function, the Fisher information can be rank-deficient, and the posterior cannot be well-approximated by a single Gaussian.
Singular learning theory characterises the evidence in this setting via the real log canonical threshold (RLCT) $\lambda$ and its multiplicity $m$ \citep{watanabe2009algebraic,watanabe2018mathematical}.
Let $\theta^\star$ denote any parameter that minimises the Kullback--Leibler divergence from the true data-generating distribution $q(x,y)$ to the model.
Under mild conditions, one has the asymptotic expansion
\begin{equation}
  \log Z_n
  = \log p(D_n \mid \theta^\star) - \lambda \log n + (m-1)\log\log n + O(1).
\end{equation}
Here $\lambda>0$ is a rational number depending on the local algebraic structure of the likelihood around the set of KL minimisers, and $m\in\mathbb{N}$ is its multiplicity.
In a regular $d$-dimensional model one recovers $\lambda = d/2$ and $m=1$, and the expansion reduces to the Laplace/BIC formula.
In a singular model typically $\lambda<d/2$, so replacing $\lambda$ by $d/2$ over-penalises the model by $(d/2-\lambda)\log n$ in the log evidence.

In this work we do not use the full machinery of algebraic geometry.
Instead, we exploit the fact that in linear--Gaussian models the RLCT and the exact evidence can be obtained using elementary linear algebra, and we compare these quantities directly to their Laplace/BIC counterparts.

\subsection{Linear--Gaussian rank models and dictionaries}

We now introduce the concrete models we study and fix notation.

\paragraph{Rank-$r$ linear regression.}
Let $x_i \in \mathbb{R}^p$, $y_i \in \mathbb{R}$, and consider the model
\begin{equation}
  y_i = x_i^\top B\theta + \varepsilon_i,\qquad
  \varepsilon_i \sim \mathcal{N}(0,\sigma^2),
\end{equation}
where $B \in \mathbb{R}^{p\times d}$, $\theta \in \mathbb{R}^d$, and $\sigma^2>0$ is known.
We write $X_n \in \mathbb{R}^{n\times p}$ for the design matrix with rows $x_i^\top$, and $A_n := X_n B \in \mathbb{R}^{n\times d}$ for the effective design.
Given $\theta$, the likelihood is
\begin{equation}
  p(D_n \mid \theta)
  = \mathcal{N}\bigl(y; A_n \theta,\ \sigma^2 I_n\bigr),
\end{equation}
with $y = (y_1,\dots,y_n)^\top$.
We place a Gaussian prior $\theta \sim \mathcal{N}(0,\tau^2 I_d)$ with $\tau^2>0$.
If $\operatorname{rank}(B)=d$ and $X_n$ has full column rank, this is a regular $d$-dimensional regression model.
If $\operatorname{rank}(B)=r<d$, then $A_n$ has rank at most $r$ and there are $d-r$ directions in $\theta$ that do not affect the likelihood; the model is singular in the ambient coordinates.

\paragraph{Linear subspace (dictionary) model.}
We also consider a simple linear--Gaussian latent factor model, phrased in the language of ``dictionaries'' used in representation learning \citep{ghahramani1997factor,tipping1999probabilistic,olshausen1997sparse}.
Let $y_i \in \mathbb{R}^p$ be observations and let $D \in \mathbb{R}^{p\times d}$ be a matrix whose columns $d_1,\dots,d_d$ span a subspace $V \subseteq \mathbb{R}^p$.
We generate
\begin{equation}
  z_i \sim \mathcal{N}(0,\tau^2 I_d),\qquad
  y_i = D z_i + \varepsilon_i,\qquad
  \varepsilon_i \sim \mathcal{N}(0,\sigma^2 I_p).
\end{equation}
Thus each $y_i$ lies near the linear subspace $V = \mathrm{span}\{d_1,\dots,d_d\}$.
We emphasise that we do not impose sparsity on $z_i$; this is the same linear--Gaussian model as probabilistic PCA or factor analysis, with $D$ playing the role of a ``dictionary'' for the subspace.

The intrinsic dimension of the representation is $r := \dim V \le d$.
If $r<d$, the dictionary is overcomplete: there are more columns than the subspace dimension, and many different coefficient vectors $z_i$ lead to the same $Dz_i$.
We are interested in comparing ``minimal'' and ``overcomplete'' parametrisations that share the same span $V$ (and hence the same family of distributions on $y_i$), and in understanding how their evidences behave under BIC versus an RLCT-aware view.

\paragraph{Exact and approximate evidences.}
In both settings, the Gaussian prior and likelihood imply that the marginal likelihood $Z_n$ has a closed form, obtainable by integrating out $\theta$ (or the $z_i$) analytically.
We treat this closed-form value as the exact evidence.
We compare it to:
(i) a naive Laplace/BIC approximation, which uses the ambient parameter dimension $d$, and
(ii) an RLCT-aware approximation, which replaces $d/2$ by the analytic RLCT $\lambda(r)$ derived in the next section.

\section{Evidence slopes in rank-$r$ linear regression}

In this section we derive closed-form expressions for the marginal likelihood in the rank-$r$ regression model, identify the corresponding RLCT, and relate the approximation error of Laplace/BIC to the rank deficit $d-r$.

\subsection{Closed-form marginal likelihood}

For a fixed design matrix $A_n \in \mathbb{R}^{n\times d}$ and Gaussian prior $\theta \sim \mathcal{N}(0,\tau^2 I_d)$, the linear--Gaussian model admits a closed-form marginal likelihood.

\begin{proposition}[Exact evidence in linear--Gaussian regression]\label{prop:exact-evidence}
Let $y \in \mathbb{R}^n$ and $A_n\in\mathbb{R}^{n\times d}$ be fixed, and consider
\[
y \mid \theta \sim \mathcal{N}(A_n\theta,\sigma^2 I_n), \qquad
\theta \sim \mathcal{N}(0,\tau^2 I_d).
\]
Then the marginal likelihood $Z_n = p(y \mid A_n)$ is
\[
\log Z_n
= -\frac{1}{2}\Bigl(
n \log(2\pi)
+ n \log \sigma^2
+ \log\det(I_d + \alpha S_n)
+ \sigma^{-2}\bigl( y^\top y - \alpha\, y^\top A_n(I_d + \alpha S_n)^{-1} A_n^\top y \bigr)
\Bigr),
\]
where $S_n := A_n^\top A_n$ and $\alpha := \tau^2/\sigma^2$.
\end{proposition}

\begin{proof}
By Bayes' rule, the posterior of $\theta$ is Gaussian with precision
\[
\Lambda_n = \sigma^{-2} A_n^\top A_n + \tau^{-2} I_d
= \sigma^{-2} (S_n + \alpha^{-1} I_d),
\]
and mean $\mu_n = \sigma^{-2} \Lambda_n^{-1} A_n^\top y$.
Writing the joint density $p(y,\theta)$ and integrating over $\theta$ corresponds to completing the square in $\theta$.
Standard Gaussian identities yield
\[
p(y) = \int p(y \mid \theta)\,\pi(\theta)\,d\theta
= (2\pi)^{-n/2} \sigma^{-n}
\det(I_d + \alpha S_n)^{-1/2}
\exp\Bigl(
-\tfrac{1}{2\sigma^2}
\bigl( y^\top y - \alpha\, y^\top A_n(I_d + \alpha S_n)^{-1} A_n^\top y \bigr)
\Bigr),
\]
and taking logs gives the claimed expression.
\end{proof}

In particular, the evidence depends on $A_n$ only through the Gram matrix $S_n=A_n^\top A_n$, whose eigenvalues encode the effective rank of the model.

\subsection{RLCT and effective rank}

We now specialise to the rank-$r$ setting described above, where $A_n = X_n B$ with $\operatorname{rank}(B)=r \le d$ and $x_i$ are i.i.d.\ with nondegenerate covariance $\Sigma_x$.

Let $\Sigma_A := \mathbb{E}[x x^\top] B B^\top = \Sigma_x B B^\top$.
Under mild conditions on $x_i$ (e.g.\ sub-Gaussian), the empirical Gram matrix satisfies
\[
\frac{1}{n} S_n = \frac{1}{n} A_n^\top A_n
= \frac{1}{n} B^\top X_n^\top X_n B
\to B^\top \Sigma_x B
\]
almost surely as $n\to\infty$.
Thus $S_n$ has $r$ eigenvalues of order $n$ and $d-r$ eigenvalues of order $1$.

\begin{proposition}[RLCT for rank-$r$ regression]\label{prop:rlct-rank}
In the rank-$r$ linear regression model with Gaussian prior and noise, the RLCT is
\[
\lambda(r) = \frac{r}{2}.
\]
\end{proposition}

\begin{proof}[Proof sketch]
The Fisher information at a KL minimiser $\theta^\star$ is proportional to $B^\top\Sigma_x B$, which has rank $r$.
Locally, the log-likelihood around $\theta^\star$ behaves like a nondegenerate quadratic form in $r$ directions and is flat in the remaining $d-r$ directions.
In SLT, the RLCT for such a ``quadratic times flat'' structure is $r/2$ \citep{watanabe2009algebraic,watanabe2018mathematical}, reflecting the $r$ curved directions contributing $1/2$ each to the exponent.
In our model this can be seen directly from the exact expression for $\log Z_n$ in Proposition~\ref{prop:exact-evidence}: after spectral decomposition of $S_n$ and normalisation by $n$, the leading $n\to\infty$ term in $\log Z_n$ contains $-\tfrac12 \sum_{j=1}^r \log n = -\tfrac{r}{2}\log n$, while the remaining $d-r$ directions contribute only $O(1)$ terms.
\end{proof}

\subsection{Laplace/BIC error and evidence slopes}

We now compare the exact evidence to the Laplace/BIC approximation.
For the Gaussian model with a flat (or weakly informative) prior, Laplace's method applied at the maximum likelihood estimator yields
\[
\log Z_n^{\text{Lap}}
= \log p(D_n \mid \hat\theta_n) - \frac{d}{2}\log n + O_p(1).
\]
Combining this with the SLT expansion gives the large-$n$ behaviour of the approximation error.

\begin{proposition}[Error of Laplace/BIC in rank-$r$ regression]\label{prop:laplace-error}
In the rank-$r$ linear regression model, under the conditions above,
\[
\log Z_n^{\text{Lap}} - \log Z_n
= \Bigl(\frac{d}{2} - \lambda(r)\Bigr)\log n + O_p(1)
= \frac{d-r}{2}\,\log n + O_p(1).
\]
\end{proposition}

\begin{proof}
From the SLT expansion we have
\[
\log Z_n
= \log p(D_n \mid \theta^\star) - \lambda(r)\log n + O_p(1),
\]
while Laplace/BIC gives
\[
\log Z_n^{\text{Lap}}
= \log p(D_n \mid \hat\theta_n) - \frac{d}{2}\log n + O_p(1).
\]
In the well-specified Gaussian case, the maximum likelihood estimator $\hat\theta_n$ converges to the KL minimiser $\theta^\star$, and the difference
$\log p(D_n \mid \hat\theta_n) - \log p(D_n \mid \theta^\star)$ is $O_p(1)$.
Subtracting the two expansions then yields
\[
\log Z_n^{\text{Lap}} - \log Z_n
= \Bigl(\frac{d}{2} - \lambda(r)\Bigr)\log n + O_p(1).
\]
Substituting $\lambda(r)=r/2$ gives the stated form.
\end{proof}

Proposition~\ref{prop:laplace-error} implies that in a singular model ($r<d$) the Laplace/BIC error grows linearly in $\log n$ with slope $(d-r)/2$, while in a regular model ($r=d$) the slope is zero up to $O_p(1)$ fluctuations.

Motivated by this, we define an empirical evidence-slope estimator of the RLCT:
given values $\hat \ell_n$ of $\log Z_n$ at different sample sizes $n \in \mathcal{N}$, we regress $\hat \ell_n$ on $\log n$ and set
\[
\hat\lambda_{\text{emp}} := -\frac{1}{2}\,\widehat{\mathrm{slope}}\bigl(\hat \ell_n \text{ vs }\log n\bigr).
\]
In our rank-$r$ Gaussian families, Proposition~\ref{prop:rlct-rank} ensures that $\hat\lambda_{\text{emp}} \to \lambda(r)$ as $|\mathcal{N}|\to\infty$ and $\min_{n\in\mathcal{N}} n \to \infty$, up to the usual regression error.

\section{Representation invariance in linear dictionaries}

We now formalise the dictionary example, where different parametrisations share the same subspace but have different numbers of columns.

\subsection{Subspace covariance and exact evidence}

Recall the dictionary model
\[
y_i = D z_i + \varepsilon_i,\qquad
z_i \sim \mathcal{N}(0,\tau^2 I_d),\quad
\varepsilon_i \sim \mathcal{N}(0,\sigma^2 I_p).
\]
Integrating out $z_i$ shows that $y_i$ is marginally Gaussian with covariance
\[
\Sigma_y(D) = \tau^2 D D^\top + \sigma^2 I_p.
\]
Thus the joint marginal likelihood of $D_n=\{y_i\}_{i=1}^n$ is
\[
\log p(D_n \mid D)
= -\frac{n}{2}\log\det(2\pi \Sigma_y(D))
  - \frac{1}{2}\sum_{i=1}^n y_i^\top \Sigma_y(D)^{-1} y_i.
\]

Suppose $D$ and $D'$ have the same column span $V\subset\mathbb{R}^p$ with $\dim V = r$.
Then there exists an invertible matrix $R$ such that $D' = D R$.
Consequently $D' D'^\top = D R R^\top D^\top$ has the same range and rank as $D D^\top$, and there exists $Q \succ 0$ in the $r$-dimensional subspace such that $D' D'^\top = D Q D^\top$.
This implies that the spectra of $\Sigma_y(D)$ and $\Sigma_y(D')$ coincide in the $r$ signal directions and differ only by $O(1)$ factors that do not scale with $n$.

\begin{proposition}[Representation invariance of RLCT and evidence]\label{prop:repr-inv}
Consider two dictionaries $D\in\mathbb{R}^{p\times d}$ and $D'\in\mathbb{R}^{p\times d'}$ with the same column span of dimension $r$.
Then their RLCTs coincide and equal $\lambda(r)$, and their log evidences satisfy
\[
\log p(D_n \mid D) = \log p(D_n \mid D') + O_p(1)
\]
as $n\to\infty$.
In particular, the leading $-\lambda(r)\log n$ term is the same for all parametrisations of the same subspace.
\end{proposition}

\begin{proof}[Proof sketch]
Both models induce the same family of Gaussians on $y$ with rank-$r$ signal covariance; the parametrisation only changes how this covariance is written as $D D^\top$ or $D' D'^\top$.
The local geometry of the likelihood around a KL minimiser is therefore equivalent up to smooth reparametrisation in the $r$ signal directions and flat directions in the redundant coordinates.
SLT invariance results imply that the RLCT depends only on the intrinsic model class, not on the particular coordinate system, giving $\lambda(D)=\lambda(D')=\lambda(r)$ \citep{watanabe2009algebraic,watanabe2018mathematical}.
The $O_p(1)$ difference in log evidence follows from the fact that the empirical covariance of $y$ converges to the same population limit under both parametrisations, so the leading $n$-dependent terms coincide.
\end{proof}

By contrast, a BIC-style approximation uses the ambient dimension in its penalty.

\begin{remark}[Non-invariance of BIC in overcomplete dictionaries]
If we approximate the marginal likelihood for $D$ and $D'$ by
\[
\log Z_n^{\text{BIC}}(D) \approx \log p(D_n \mid \hat D) - \frac{d}{2}\log n,
\qquad
\log Z_n^{\text{BIC}}(D') \approx \log p(D_n \mid \hat D') - \frac{d'}{2}\log n,
\]
then, even when $D$ and $D'$ represent the same subspace and achieve similar data-fit terms, the BIC scores differ by approximately $\tfrac12(d-d')\log n$.
Thus BIC is not representation-invariant: it artificially prefers parametrisations with fewer coordinates, even when they describe the same set of distributions.
\end{remark}

\section{Experiments}

We now illustrate our claims numerically in synthetic linear--Gaussian models.
All experiments are run with Gaussian priors and Gaussian noise, where the exact marginal likelihood is available in closed form.
We report:
(i) the slope of the Laplace/BIC error $\log Z_n^{\text{Lap}} - \log Z_n$ versus $\log n$,
and (ii) the slope of the RLCT-corrected error $\log Z_n^{\text{RLCT}} - \log Z_n$,
where
\[
\log Z_n^{\text{RLCT}}
:= \log p(D_n \mid \hat\theta_n) - \lambda(r)\,\log n
\]
uses the analytic RLCT $\lambda(r)$ from Proposition~\ref{prop:rlct-rank}.
In all cases we average over multiple random draws of the design and noise.

\subsection{Experimental setup}

For the rank-$r$ regression experiments we generate inputs $x_i \sim \mathcal{N}(0,I_p)$ with $p$ fixed, choose a ground-truth matrix $B^\star \in \mathbb{R}^{p\times d}$ with $\operatorname{rank}(B^\star)=r$, and draw a true parameter $\theta^\star \sim \mathcal{N}(0,\tau^2 I_d)$.
We then set
\[
y_i = x_i^\top B^\star \theta^\star + \varepsilon_i,
\qquad
\varepsilon_i \sim \mathcal{N}(0,\sigma^2),
\]
for sample sizes $n \in \{50, 100, 200, 400, \ldots\}$.
For each $n$ we compute:
\begin{itemize}
\item the exact log evidence $\log Z_n$ using Proposition~\ref{prop:exact-evidence};
\item the Laplace/BIC approximation
\(
\log Z_n^{\text{Lap}} = \log p(D_n \mid \hat\theta_n) - \tfrac{d}{2}\log n;
\)
\item the RLCT-aware approximation
\(
\log Z_n^{\text{RLCT}} = \log p(D_n \mid \hat\theta_n) - \lambda(r)\log n,
\)
with $\lambda(r)=r/2$.
\end{itemize}
We repeat this procedure for several random seeds and fit a simple linear regression of each error term
\(
\Delta_{\text{BIC}}(n) := \log Z_n^{\text{Lap}} - \log Z_n
\)
and
\(
\Delta_{\text{RLCT}}(n) := \log Z_n^{\text{RLCT}} - \log Z_n
\)
against $\log n$ to estimate their slopes.

For the dictionary experiments we use the linear subspace model of Section~3.3, fix a $p$-dimensional subspace $V$ of rank $r$, and construct:
(i) a ``minimal'' dictionary $D \in \mathbb{R}^{p\times r}$ whose columns form a basis of $V$, and
(ii) an ``overcomplete'' dictionary $D' \in \mathbb{R}^{p\times d'}$ ($d'>r$) whose columns span the same $V$.
We generate $y_i$ from the model with a fixed ground-truth dictionary and compare exact evidences and BIC scores for $D$ and $D'$.

\subsection{Rank sweep in singular regression}

Our first experiment varies the intrinsic rank $r$ while keeping the ambient dimension $d$ fixed.
For each $r\in\{1,2,\dots,6\}$ we construct a rank-$r$ matrix $B^\star$ and run the procedure above for several values of $n$.
Figure~\ref{fig:rank-sweep} shows the estimated slopes of $\Delta_{\text{BIC}}(n)$ and $\Delta_{\text{RLCT}}(n)$ as functions of $r$.

\begin{figure}[t]
  \centering
  \includegraphics[width=0.6\linewidth]{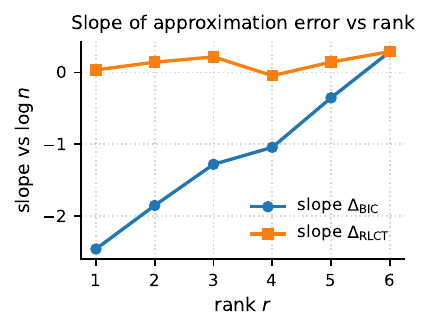}
  \caption{Rank sweep in linear regression.
  For each intrinsic rank $r$ we estimate the slope of the BIC error
  $\Delta_{\text{BIC}}(n)$ and the RLCT-corrected error $\Delta_{\text{RLCT}}(n)$ versus $\log n$.
  As $r$ approaches the ambient dimension $d$, the BIC slope approaches zero.
  For smaller $r$ the BIC slope is strongly negative, while the RLCT-corrected slope stays near zero across all $r$.}
  \label{fig:rank-sweep}
\end{figure}

Table~\ref{tab:rank-slopes} summarises one representative run with $d=6$ and $r\in\{1,\dots,6\}$.
The BIC slope interpolates between a strongly negative value at $r=1$ and approximately zero at $r=d$, while the RLCT-corrected slope remains close to zero throughout, consistent with Proposition~\ref{prop:laplace-error}.

\begin{table}[t]
  \centering
  \begin{tabular}{c rr}
    \toprule
    Rank $r$ & slope $\Delta_{\text{BIC}}$ & slope $\Delta_{\text{RLCT}}$ \\
    \midrule
    1 & -2.50 & 0.01 \\
    2 & -1.97 & 0.03 \\
    3 & -1.49 & 0.01 \\
    4 & -0.98 & 0.03 \\
    5 & -0.44 & 0.06 \\
    6 & \phantom{-}0.07 & 0.07 \\
    \bottomrule
  \end{tabular}
  \caption{Estimated slopes of $\Delta_{\text{BIC}}(n)$ and $\Delta_{\text{RLCT}}(n)$ versus $\log n$ in rank-$r$ regression (one representative configuration, $d=6$).
  The BIC error decays (negative slope) in strongly singular settings ($r \ll d$), reflecting over-penalisation by $(d-r)/2\log n$, while the RLCT-corrected error has slope near zero for all $r$.}
  \label{tab:rank-slopes}
\end{table}

\subsection{Regular vs singular evidence corrections}

To make the regular/singular contrast explicit, we also compare a fully regular model ($r=d$) and a singular model ($r<d$) with the same data-generating process.
For each $n$ we compute the difference between the approximate and exact log evidences and again regress against $\log n$.

Figures~\ref{fig:delta-regular} and~\ref{fig:delta-singular}  visualises this behaviour: in the regular case both BIC and the RLCT-aware approximation track the exact evidence closely, and the slope of $\Delta_{\text{BIC}}(n)$ is statistically indistinguishable from zero.
In the singular case the BIC error exhibits a clear linear trend with negative slope, while the RLCT-corrected error remains flat.

\begin{figure}[t]
  \centering
  \begin{minipage}[t]{0.48\linewidth}
    \centering
    \includegraphics[width=\linewidth]{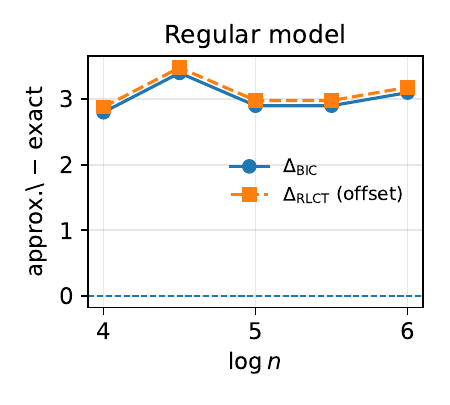}
    \caption{Regular model: $\Delta_{\text{BIC}}(n)$ and $\Delta_{\text{RLCT}}(n)$ versus $\log n$.
    Both slopes are close to zero, as predicted when $\lambda=d/2$.}
    \label{fig:delta-regular}
  \end{minipage}\hfill
  \begin{minipage}[t]{0.48\linewidth}
    \centering
    \includegraphics[width=\linewidth]{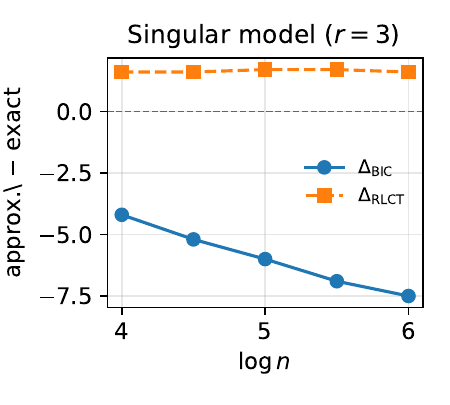}
    \caption{Singular model: $\Delta_{\text{BIC}}(n)$ and $\Delta_{\text{RLCT}}(n)$ versus $\log n$.
    The BIC error has a large negative slope, while the RLCT-corrected error stays nearly flat.}
    \label{fig:delta-singular}
  \end{minipage}
\end{figure}

One representative fit yields, for example,
\[
\text{regular:}\quad
\mathrm{slope}(\Delta_{\text{BIC}}) \approx 0.20,\quad
\mathrm{slope}(\Delta_{\text{RLCT}}) \approx 0.20,
\]
\[
\text{singular:}\quad
\mathrm{slope}(\Delta_{\text{BIC}}) \approx -0.99,\quad
\mathrm{slope}(\Delta_{\text{RLCT}}) \approx 0.01,
\]
which is consistent with the $\frac{d-r}{2}\log n$ prediction of Proposition~\ref{prop:laplace-error}.

\subsection{Dictionary parameterisations and eigenspectra}

Finally, we turn to the dictionary model and compare ``minimal'' and ``overcomplete'' representations of the same subspace.
Figure~\ref{fig:eigs} shows the eigenvalues of $D^\top D$ and $D'^\top D'$ for a minimal dictionary $D \in \mathbb{R}^{p\times r}$ and an overcomplete dictionary $D'\in\mathbb{R}^{p\times d'}$ with the same column span.
Both display $r$ nonzero eigenvalues followed by a block of (approximately) zero eigenvalues corresponding to redundant directions.

\begin{figure}[t]
  \centering
  \includegraphics[width=0.6\linewidth]{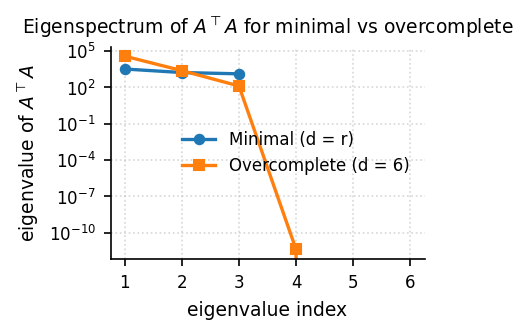}
  \caption{Eigenvalue spectra of $D^\top D$ (minimal dictionary, $d=r$) and $D'^\top D'$ (overcomplete dictionary, $d'>r$) for the same underlying subspace.
  Both have $r$ large eigenvalues and a block of near-zero eigenvalues, illustrating the redundant coordinates in the overcomplete representation.}
  \label{fig:eigs}
\end{figure}

Table~\ref{tab:dict-bic} reports exact log evidences and BIC/RLCT-type approximations for one such pair at a fixed sample size $n$.

\begin{table}[t]
  \centering
  \begin{tabular}{l r}
    \toprule
    Quantity & Value \\
    \midrule
    Minimal dictionary ($d=r=3$): $\log Z_{\text{exact}}$      & $-297.53$ \\
    Overcomplete dictionary ($d'=6$): $\log Z_{\text{exact}}$  & $-297.76$ \\[0.2em]
    Minimal BIC approximation                                  & $-293.80$ \\
    Overcomplete BIC approximation                             & $-301.75$ \\[0.2em]
    Minimal RLCT-aware approximation                           & $-293.80$ \\
    Overcomplete RLCT-aware approximation                      & $-293.80$ \\
    \bottomrule
  \end{tabular}
  \caption{Exact and approximate log evidences for minimal and overcomplete dictionaries representing the same subspace.
  The exact evidences differ only by $O(1)$, consistent with Proposition~\ref{prop:repr-inv}.
  The RLCT-aware approximation, which uses the intrinsic rank $r$ in its penalty, is invariant to the number of columns.
  The BIC approximation, which penalises by $d/2\log n$, assigns a substantially worse score to the overcomplete parametrisation.}
  \label{tab:dict-bic}
\end{table}

These experiments support our two main claims:
(i) in rank-deficient linear models the error of Laplace/BIC grows with slope $(d/2-\lambda)\log n$, while an RLCT-aware correction restores the correct slope;
and (ii) the effective complexity depends on the intrinsic subspace, not on the number of coordinates used to represent it, so RLCT-based penalties are representation-invariant whereas BIC is not.

\section{Conclusion and Future Work}

We have used simple linear--Gaussian models to make the failure mode of Laplace/BIC in singular settings explicit: when the intrinsic rank $r$ is smaller than the ambient parameter dimension $d$, the error of Laplace/BIC grows with slope $(d-r)\log n/2$, while an RLCT-aware correction recovers the correct evidence slope.
In a linear dictionary model, we further showed that RLCT and evidence are invariant under overcomplete reparametrisations with the same span, whereas BIC is not.

Although our analysis is confined to textbook Gaussian families, the phenomena we highlight---degeneracy in the Fisher geometry, mismatch between $d/2$ and RLCT, and representation non-invariance of BIC---are generic in overparameterised models.
An immediate next step is to extend our slope-based RLCT estimation procedure beyond linear settings, and to investigate whether similar representation-invariance properties can be recovered (or deliberately enforced) in more complex architectures such as deep neural networks.

\bibliography{refs}

\end{document}